\pdfoutput=1
\documentclass[11pt]{article}

\usepackage{times}

\usepackage{amsmath}
\usepackage{amsfonts}
\usepackage{amssymb}
\usepackage{xcolor}
\usepackage{amsmath,mathtools}
\usepackage{algorithm,algorithmic}
\usepackage{graphicx}
\usepackage{adjustbox}
\usepackage{amsthm} 
\usepackage[square,numbers]{natbib}

\newcounter{lastnote}

\title{\bf Unconstrained Online Optimization: Dynamic Regret Analysis of Strongly Convex and Smooth Problems} 
\author
{Ting-Jui Chang and Shahin Shahrampour\\
\\
\normalsize{Texas A\&M University\vspace{.2cm}}\\
\normalsize{E-mail: {\tt tingjui.chang@tamu.edu, shahin@tamu.edu}}
}

\date{}

\newcommand{\mb}{\mathbf{m}}

\newcommand{\vb}{\mathbf{v}}
\newcommand{\ub}{\mathbf{u}}

\newcommand{\xb}{\mathbf{x}}
\newcommand{\yb}{\mathbf{y}}
\newcommand{\zb}{\mathbf{z}}

\newcommand{\Ab}{\mathbf{A}}

\newcommand{\Hb}{\mathbf{H}}
\newcommand{\Ib}{\mathbf{I}}

\newcommand{\Mb}{\mathbf{M}}

\newcommand{\Xc}{\mathcal{X}}

\newcommand{\argmin}{\text{argmin}}

\newcommand{\norm}[1]{\left\lVert#1\right\rVert}

\newtheorem{theorem}{Theorem}

\newtheorem{corollary}[theorem]{Corollary}
\newtheorem{definition}{Definition}

\newtheorem{lemma}[theorem]{Lemma}

\begin{document} 


\baselineskip24pt


\maketitle 


\begin{abstract}
    The regret bound of {\it dynamic} online learning algorithms is often expressed in terms of the variation in the function sequence ($V_T$) and/or the path-length of the minimizer sequence after $T$ rounds. For strongly convex and smooth functions, \citet{zhang2017improved} establish the squared path-length of the minimizer sequence ($C^*_{2,T}$) as a lower bound on regret. They also show that online gradient descent (OGD) achieves this lower bound using multiple gradient queries per round. In this paper, we focus on unconstrained online optimization. We first show that a preconditioned variant of OGD achieves $O(C^*_{2,T})$ with one gradient query per round. We then propose online optimistic Newton (OON) method for the case when the first and second order information of the function sequence is {\it predictable}. The regret bound of OON is captured via the quartic path-length of the minimizer sequence ($C^*_{4,T}$), which can be much smaller than $C^*_{2,T}$. We finally show that by using multiple gradients for OGD, we can achieve an upper bound of $O(\min\{C^*_{2,T},V_T\})$ on regret.
\end{abstract}

\section{Introduction}
Online optimization is modeled as a repeated game between a learner and an adversary \citep{hazan2016introduction}. At the $t$-th round, $t\in [T]\triangleq\{1,\ldots,T\}$, the learner selects an action $\xb_t$ from a convex set $\Xc \subseteq \mathrm{R}^n$ based on the information from previous rounds. Then, the adversary reveals a convex function $f_t:\Xc\to\mathrm{R}$ to the learner that incurs the loss $f_t(\xb_t)$. The goal of online learning is to minimize the \emph{regret}, which is the difference between the cumulative loss of the learner and that of a comparator sequence in hindsight. Depending on the comparator sequence, the regret can be either static or dynamic. The static regret is defined with respect to a fixed comparator as follows
\begin{align}\label{eq:static}
\mathbf{Reg}^s_T \triangleq \sum^T_{t=1}f_t(\xb_t)-\min_{\xb\in\Xc}\sum^T_{t=1}f_t(\xb).
\end{align}

Static regret is well-studied in the literature of online optimization. \citet{zinkevich2003online} shows that online gradient descent (OGD) provides a $O(\sqrt{T})$ upper bound on static regret for convex functions. \citet{hazan2007logarithmic} improve this bound to $O(\log T)$ for exp-concave functions as well as strongly convex functions. These bounds turn out to be optimal given their corresponding lower bounds \citep{hazan2016introduction}.

A more stringent benchmark for regret can be defined when the comparator sequence is {\it time-varying}, introducing the notion of {\it dynamic} regret \citep{besbes2015non,jadbabaie2015online}. In this case, the learner's performance is measured against the best sequence of actions (minimizers) at each round as follows
\begin{equation}\label{eq: definition of dynamic regret}
    \mathbf{Reg}^d_T \triangleq  \sum^T_{t=1}f_t(\xb_t)-\sum^T_{t=1}f_t(\xb_t^*),
\end{equation}
where $\xb_t^*\triangleq \argmin_{\xb\in\Xc}f_t(\xb)$. More generally, dynamic regret against an arbitrary  comparator sequence $\{\ub_t\}_{t=1}^T$ is defined as \citep{zinkevich2003online}, 
\begin{equation}\label{eq: definition of dynamic regret general}
    \mathbf{Reg}^d_T(\ub_1,\ldots,\ub_T) \triangleq \sum^T_{t=1}f_t(\xb_t)-\sum^T_{t=1}f_t(\ub_t).
\end{equation}

It is well-known that since the function sequence can fluctuate arbitrarily, the worst-case dynamic regret scales linearly with respect to $T$. However, when the environment varies slowly, there is hope to bound dynamic regret. 

In the past few years, various studies on {\it dynamic online learning} have provided regret bounds in terms of the variation in the function sequence and/or the path-length of the minimizer sequence \citep{besbes2015non,hall2015online,jadbabaie2015online}. The path-length of an arbitrary sequence $\{\ub_t\}_{t=1}^T$ is defined as \citep{zinkevich2003online},
\begin{equation}\label{eq: path length}
    C_T(\ub_1,\ldots,\ub_T)\triangleq \sum_{t=2}^T\norm{\ub_t-\ub_{t-1}}.
\end{equation}
\citet{zinkevich2003online} shows that applying OGD for convex functions results in an upper bound of $O(\sqrt{T}C_T)$ on dynamic regret. If the function sequence is assumed to be strongly convex and smooth, the upper bound can be further improved to $O(C_T^*)$ \citep{mokhtari2016online}, where 
$$C_T^*\triangleq C_T(\xb_1^*,\ldots,\xb_T^*)=\sum_{t=2}^T\norm{\xb^*_t-\xb^*_{t-1}}.$$
Let us now define a new variation measure $C_{p,T}^*$ ( the path-length of order $p$) as follows
\begin{equation}\label{eq: path length order p}
     C_{p,T}^*\triangleq C_{p,T}(\xb_1^*,\ldots,\xb_T^*)= \sum_{t=2}^T\norm{\xb_t^*-\xb_{t-1}^*}^p,
\end{equation}
with the convention that $C_T^*=C_{1,T}^*$. It can be immediately seen that any bound of order $C_{p,T}^*$ implies a bound of order $C_{q,T}^*$ for $q<p$, as long as the minimizer sequence is assumed to be uniformly bounded. Recently, \citet{zhang2017improved} prove that by using multiple gradient queries in one round, the regret of OGD can be improved to $O(C_{2,T}^*)$, which can be much smaller than $C_T^*$ when the local variations are small. \citet{zhang2017improved} also prove that the bound $O(C_{2,T}^*)$ is optimal in the worst-case. 

Besides the path-length, another commonly used regularity measure is $V_T$, the cumulative variation in the function sequence, defined as
 \begin{equation}\label{eq: variation in funciton values}
    V_T\triangleq \sum_{t=2}^T \sup_{\xb \in \Xc} |f_t(\xb)-f_{t-1}(\xb)|.
\end{equation}
\citet{besbes2015non} show that the dynamic regret can be bounded by $O(T^{2/3}(V_T+1)^{1/3})$ and $O(\sqrt{T(1+V_T)})$
for convex functions and strongly convex functions, respectively. Note that in general $C_T$ and $V_T$ are not directly comparable, and \citet{jadbabaie2015online} provide problem environments where $V_T$ and $C_T$ are significantly different in terms of the order.  
 
 In this work, we focus on {\it unconstrained} online optimization for strongly convex and smooth functions and study dynamic regret in the sense of \eqref{eq: definition of dynamic regret}. Our contribution is threefold:
 \begin{itemize}
     \item We propose online preconditioned gradient descent (OPGD), where the gradient direction is re-scaled by a time-varying positive-definite matrix at each round. We show that OPGD achieves the dynamic regret bound of $O(C^*_{2,T})$ with one gradient query in each round. Beside matching the lower bound in \citep{zhang2017improved}, the result also entails that OGD and a regularized variant of online Newton method enjoy the same regret bound (Section \ref{sec:opgd}).  
     \item Inspired by optimistic mirror descent \citep{rakhlin2013optimization}, where predictions of the gradient sequence are used, we propose optimistic online Newton (OON) by incorporating predictions of the Hessian and gradient to online Newton method. We prove that in this case the dynamic regret bound can be further improved to $O(C^*_{4,T})$ if the predictions are accurate enough (Section \ref{sec:newton}).  
     \item We finally show that by applying multiple gradient descents, the dynamic regret is upper bounded by $O(\min \{V_T, C^*_{2,T}\})$. We also construct problem setups where $V_T$ is much larger than $C^*_{2,T}$ in order, and vice versa (Section \ref{sec:4}).  
 \end{itemize}
The proofs of our results are provided in the supplementary material. 

\section{Related Literature}
{ 
\renewcommand{\arraystretch}{1.3}
\begin{table}[h!]
\caption{Related works on dynamic online learning (single gradient query in each round)}
\centering
\begin{adjustbox}{max width=0.99\columnwidth}
\begin{tabular}{||c c c c||} 
 \hline
 Reference & Regret Definition & Setup & Regret Bound\\ 
 \hline\hline
 \citep{zinkevich2003online} & $\sum_{t=1}^T f_t(\xb_t)-f_t(\ub_t)$ & Convex & $O\big(\sqrt{T}(1+C_T(\ub_1,\ldots,\ub_T))\big)$ \\ 
 \citep{hall2015online} & $\sum_{t=1}^T f_t(\xb_t)-f_t(\ub_t)$ & Convex & $O\big(\sqrt{T}(1+C^\prime_T(\ub_1,\ldots,\ub_T))\big)$\\
 \citep{besbes2015non} & $\sum_{t=1}^T \mathbf{E}[f_t(\xb_t)]-f_t(\xb_t^*)$ & Convex & $O\big(T^{2/3}(1+V_T)^{1/3}\big)$\\
 \citep{besbes2015non} & $\sum_{t=1}^T \mathbf{E}[f_t(\xb_t)]-f_t(\xb_t^*)$ & Strongly Convex & $O\big(\sqrt{T(1+V_T)}\big)$\\
 \citep{jadbabaie2015online} & $\sum_{t=1}^T f_t(\xb_t)-f_t(\xb_t^*)$ & Convex & $O\big(\sqrt{D_T+1}+\min \big\{\sqrt{(D_T+1)C_T^*}, [(D_T+1)V_T T]^{1/3}\big\}\big)$\\
 \citep{mokhtari2016online} & $\sum_{t=1}^T f_t(\xb_t)-f_t(\xb_t^*)$ & Strongly Convex and Smooth & $O\big(C_T^*\big)$\\
 \citep{zhang2018adaptive} & $\sum_{t=1}^T f_t(\xb_t)-f_t(\ub_t)$ & Convex & $O\big(\sqrt{T(1+C_T(\ub_1,\ldots,\ub_T))}\big)$ \\
 \textbf{This work} & $\sum_{t=1}^T f_t(\xb_t)-f_t(\xb_t^*)$ & Strongly Convex and Smooth & ~~~~~~~~~~~~~~~~~~~~~~~~~~~~~~~~~~~~ $O\big(C_{2,T}^*\big)$ ~~~~~~~~~~~~~~ {\tt OPGD Algorithm}\\
 \textbf{This work} & $\sum_{t=1}^T f_t(\xb_t)-f_t(\xb_t^*)$ & Strongly Convex and Smooth & ~~~~~~~~~~~~~~~~~~~~~~~~~~~~~~~~~~~~ $O\big(C_{4,T}^* + D^\prime_T\big)$ ~~~~~ {\tt OON Algorithm}\\
 \hline
\end{tabular}
\end{adjustbox}
\label{table:Dynamic regret bounds of related works 1}
\end{table}
}

{ 
\renewcommand{\arraystretch}{1.3}
\begin{table}[h!]
\caption{Related works on dynamic online learning (multiple gradient queries in each round)}
\centering
\begin{adjustbox}{max width=0.99\columnwidth}
\begin{tabular}{||c c c c||} 
 \hline
 Reference & Regret Definition & Setup & Regret Bound\\ 
 \hline\hline
 \citep{zhang2017improved} & $\sum_{t=1}^T f_t(\xb_t)-f_t(\xb_t^*)$ & Strongly Convex and Smooth & $O\big(\min\{C^*_T,C_{2,T}^*\}\big)$\\
 \textbf{This work} & $\sum_{t=1}^T f_t(\xb_t)-f_t(\xb_t^*)$ & Strongly Convex and Smooth & $O\big(\min\{V_T,C_{2,T}^*\}\big)$\\
 \hline
\end{tabular}
\end{adjustbox}
\label{table:Dynamic regret bounds of related works 2}
\end{table}
}

In this section, we provide related literature on dynamic regret defined in \eqref{eq: definition of dynamic regret} and \eqref{eq: definition of dynamic regret general}. A summary of the results is tabulated in Tables \ref{table:Dynamic regret bounds of related works 1} and \ref{table:Dynamic regret bounds of related works 2}. In particular, Table \ref{table:Dynamic regret bounds of related works 1} summarizes results with one gradient query per round, whereas Table \ref{table:Dynamic regret bounds of related works 2} exhibits those with multiple gradient queries per round.

As previously mentioned, \citet{zinkevich2003online} shows that when the functions are convex, by applying OGD with a diminishing step size of $1/\sqrt{t}$, the dynamic regret defined in \eqref{eq: definition of dynamic regret general} can be bounded by $O(\sqrt{T}(1+C_T))$. \citet{zhang2018adaptive} combine OGD with expert advice to improve the bound to $O(\sqrt{T(1+C_T)})$. Focusing on 
regret in the sense of \eqref{eq: definition of dynamic regret}, \citet{mokhtari2016online} establish a regret bound of $O(C^*_T)$ for OGD under strong convexity and smoothness of the function sequence. \citet{lesage2020second} show the same bound for online Newton method. 

To further express the existing regret bounds, we need to define several other regularity measures. The first one is similar to the path-length \eqref{eq: path length} and is defined as
\begin{equation}\label{eq: path length prime}
    C^\prime_T(\ub_1,\ldots,\ub_T)\triangleq\sum_{t=2}^T\norm{\ub_t - \Phi_t(\ub_{t-1})},
\end{equation}
where $\Phi_t(\cdot)$ is a given dynamics (available to the learner). 
\citet{hall2015online} propose a dynamic mirror descent algorithm that incorporates the dynamics $\{\Phi_t(\cdot)\}_{t=1}^T$ into online mirror descent and achieves a regret bound of $O(\sqrt{T}(1+C^\prime_T))$. 

\citet{besbes2015non} propose a restarted OGD and analyze its performance for the case when only {\it noisy} gradients are available to the learner. They prove that the expected dynamic regret is bounded by $O(T^{2/3}(V_T+1)^{1/3})$ and $O(\sqrt{T(1+V_T)})$ for convex and strongly convex functions, respectively. The restarted OGD of \citep{besbes2015non} is designed under the assumption that $V_T$ (or an upper bound on $V_T$) is available to the learner from the outset.

 Another measure is $D_T$, the variation in gradients, which is defined as
\begin{equation}\label{eq: variation in gradients}
    D_T\triangleq \sum_{t=1}^T \norm{\nabla f_t(\xb_t)-\mb_t}^2,
\end{equation}
where $\mb_t$ is a predictable sequence computed by the learner before round $t$ \citep{rakhlin2013online, rakhlin2013optimization}. A special version of $D_T$ with $\mb_t=\nabla f_{t-1}$ is introduced by \citep{chiang2012online}, and the current definition is used by \citep{rakhlin2013online, rakhlin2013optimization} for studying optimistic mirror descent. Nevertheless, all of these works deal with static regret. Motivated by the fact that various regularity measures are not directly comparable, \citet{jadbabaie2015online} propose an adaptive version of optimistic mirror descent to bound dynamic regret \eqref{eq: definition of dynamic regret}. They establish a regret bound in terms of $C_T$, $D_T$, and $V_T$ for convex functions with the assumption that the learner can accumulate each of these measures on-the-fly. When $V_T=0$ or $C_T=0$, their bound recovers that of \citep{rakhlin2013optimization} on static regret. 

 Inspired by the notion of $D_T$, we introduce in Section \ref{sec:newton} a new regularity $D^\prime_T$ defined as
\begin{equation}\label{eq: variation in Hessians and gradients}
    D^\prime_T\triangleq \sum_{t=2}^T \norm{\Mb_t^{-1}(\cdot)\mb_t(\cdot)-(\nabla^2 f_{t}(\cdot))^{-1}\nabla f_t(\cdot)}^2,
\end{equation}
where $\Mb_t(\cdot)$ and $\mb_t(\cdot)$ denote the predictions of $\nabla^2f_t(\cdot)$ and $\nabla f_t(\cdot)$, respectively. We later show that by applying OON, the dynamic regret bound is $O(D^\prime_T+C^*_{4,T})$. This is an improvement over $O(C^*_{2,T})$ only if $D^\prime_T$ is small, i.e., the predictions of the learner are accurate enough. Nevertheless, we also show that if the learner uses stale gradient/Hessian information in the form of $\mb_t=\nabla f_{t-1}$ and $\Mb_t=\nabla^2f_{t-1}$, the regret is still $O(C^*_{2,T})$. 

Other related works on dynamic regret include \citep{ravier2019prediction,yuan2019trading}. \citet{ravier2019prediction} assume the function sequence has a parametrizable structure and quantify the functional difference in terms of the variation in parameters. They propose an online gradient method combined with the prediction of the function parameters and show that the dynamic regret can be bounded in terms of $C^*_T$ as well as the accumulation error in the parameters. \citet{yuan2019trading} analyze the trade-off between static and dynamic regret through studying the effect of forgetting factors for a class of online Newton algorithms.

We note that \citet{zhang2017improved} prove a $O(\min\{C^*_T, C_{2,T}^*\})$ regret bound with multiple gradient queries for OGD. We revisit the same algorithm (in an unconstrained setup) and establish a bound of $O(\min\{V_T, C_{2,T}^*\})$. The main benefit of the latter is that $V_T$ and $C_{2,T}^*$ are not comparable, whereas $C^*_{2,T}=O(C_T^*)$ as long as the minimizer sequence is bounded. 

{\bf Adaptive Regret:} Beside the works related to the dynamic regret, the notion of adaptive regret \citep{hazan2007adaptive,daniely2015strongly,zhang2018dynamic, zhang2019adaptive,zhang2020minimizing} is also proposed to capture the dynamics in the environment. Adaptive regret characterizes a local version of static regret, where
 $$\mathbf{Reg}^a_T([r,s]) \triangleq \sum_{t=r}^s f_t(\xb_t)-\min_{\xb\in\Xc}\sum_{t=r}^s f_t(\xb),$$
 for each interval $[r,s]\subseteq [T]$. \citet{zhang2018dynamic} draw a connection between strongly adaptive regret and dynamic regret and propose an adaptive algorithm which can bound the dynamic regret without prior knowledge of the functional variation. \citet{zhang2020minimizing} propose a novel algorithm which can minimize the dynamic regret and the adaptive regret simultaneously.
\section{Main Results}
In this section, we present our main results. We first prove that OPGD achieves the optimal bound of $O(C_{2,T}^*)$ for dynamic regret \eqref{eq: definition of dynamic regret}, matching the lower bound of \citep{zhang2017improved}. Then, we develop a variant of online Newton method (called OON), which employs predicted first and second order information in the update. The bound on the dynamic regret of OON can be improved to $O(C_{4,T}^*)$ if the predictions are accurate.  

\subsection{Preliminaries}
Since our results are on strongly convex and smooth functions, we start by their formal definitions below. Throughout, we assume that the function sequence $\{f_t\}_{t=1}^T$ is differentiable. 
\begin{definition}
A function $f:\Xc\to \mathrm{R}$ is $\mu$-strongly convex ($\mu>0$) over the convex set $\Xc$ if $$f(\xb)\geq f(\yb)+\nabla f(\yb)^\top(\xb-\yb) +\frac{\mu}{2}\norm{\xb-\yb}^2,\quad \forall\xb,\yb \in \Xc.\quad.$$
\end{definition}

\begin{definition}
A function $f:\Xc\to \mathrm{R}$ is $L$-smooth ($L>0$), when its gradient is Lipschitz continuous over the convex set $\Xc$, where
$$f(\xb)\leq f(\yb)+\nabla f(\yb)^\top(\xb-\yb) +\frac{L}{2}\norm{\xb-\yb}^2,\quad \forall\xb,\yb \in \Xc.\quad$$
\end{definition}

\subsection{Dynamic Regret Bound for Online Preconditioned Gradient Descent (OPGD)}\label{sec:opgd}
When the functions are $\mu$-strongly convex and $L$-smooth, 
\citet{mokhtari2016online} prove that the dynamic regret of OGD can be upper bounded by $O(C_T^*)$. \citet{zhang2017improved} show that multiple gradient descents can help improving it to $O(\min\{C_T^*, C_{2,T}^*\})$.


We observe that in the unconstrained setup (which implies $\norm{\nabla f_t(\xb_t^*)}=0$ for every $t\in [T]$), querying just one gradient in each round is enough to obtain $O(C_{2,T}^*)$. We analyze this observation for a slightly more general case where OGD is preconditioned, i.e., the gradient direction at each round is re-scaled according to a positive definite matrix. Our method, called OPGD, is summarized in Algorithm \ref{alg:OPGD}. We establish the regret bound in the following theorem, under an assumption that characterizes the relationship of the condition number of the matrices used for preconditioning to parameters $\mu$ and $L$.

\begin{theorem}\label{T: Improved OPGD dynamic bound} Suppose that for any $t\in [T]$:
\begin{enumerate}
    \item The function $f_t:\mathrm{R}^n\to \mathrm{R}$ is $\mu$-strongly convex and $L$-smooth, and $\xb^*_t$ is bounded.
    \item The preconditioning matrix $\Ab_t$ satisfies $\lambda^\prime\cdot \Ib\preceq \Ab_t \preceq \lambda \cdot \Ib.$
    \item The condition number satisfies $\frac{\lambda}{\lambda\prime} < 1+\frac{\mu^2}{4L^2}.$
\end{enumerate}
If we set $\eta=\frac{\lambda^\prime\mu}{2L^2}$, the dynamic regret for the sequence of actions $\xb_t$ generated by OPGD is bounded as follows
$$\mathbf{Reg}^d_T\leq  \left(\frac{L^2}{\mu}\right)\left(\frac{4L^2\lambda-\mu^2\lambda^\prime}{\mu^2\lambda^\prime - 4L^2\lambda + 4L^2\lambda^\prime}\right)\sum_{t=2}^{T+1} \norm{\xb_t^* - \xb_{t-1}^*}^2 + \left(\frac{L^2\lambda}{\lambda^\prime\mu} - \frac{\mu}{4}\right)\norm{\xb_1 - \xb_1^*}^2.$$\\
\end{theorem}
The theorem above shows that OPGD achieves $O(C_{2,T}^*)$ regret. An immediate corollary is that OGD also achieves the same rate if we set $\Ab_t=\Ib$, which implies $\lambda=\lambda^\prime=1$. 

\begin{algorithm}[tb]
   \caption{Online Preconditioned Gradient Descent (OPGD)}
   \label{alg:OPGD}
\begin{algorithmic}[1]
   \STATE {\bfseries Require:} Initial vector $\xb_1\in \mathrm{R}^n$, step size $\eta$, a sequence of positive definite matrices $\Ab_t$
  
   \FOR{$t=1,2,\ldots,T$}
   \STATE Play $\xb_t$
   \STATE Observe the gradient of the current action $\nabla f_t(\xb_t)$
   \STATE $\xb_{t+1} = \xb_t - \eta\Ab^{-1}_t\nabla f_t(\xb_t)$
   \ENDFOR
\end{algorithmic}
\end{algorithm}

\begin{corollary}\label{Corollary: OGD Recovery}
Suppose that for any $t\in [T]$, the function $f_t:\mathrm{R}^n\to \mathrm{R}$ is $\mu$-strongly convex and $L$-smooth, and $\xb^*_t$ is bounded. Let $\Ab_t=\Ib$,  then OPGD amounts to OGD, and it achieves a regret bound of $O(C_{2,T}^*)$.
\end{corollary}
Another corollary of Theorem \ref{T: Improved OPGD dynamic bound} is on a regularized version of online Newton method as follows. 
\begin{corollary}\label{Corollary: ONS Recovery}
Suppose that for any $t\in [T]$, the function $f_t:\mathrm{R}^n\to \mathrm{R}$ is $\mu$-strongly convex and $L$-smooth, and $\xb^*_t$ is bounded. Let $\Ab_t=\nabla^2 f_t(\xb_t) + \zeta\cdot\Ib$ where $\zeta > \frac{(L-\mu)4L^2}{\mu^2}-\mu$. Then, OPGD corresponds to a regularized variant of online Newton method, and it achieves a regret bound of $O(C_{2,T}^*)$.
\end{corollary}
\begin{proof}
We just need to verify the third condition in Theorem \ref{T: Improved OPGD dynamic bound} for which we require $$\frac{L+\zeta}{\mu+\zeta}<1+\frac{\mu^2}{4L^2}\Longleftrightarrow \zeta > \frac{(L-\mu)4L^2}{\mu^2}-\mu.$$
\end{proof}
Compared to the regret bound of \citep{lesage2020second} for online Newton method, Corollary \ref{Corollary: ONS Recovery} puts no constraint on the relative location of the starting point, i.e., the result is global (and not local). Moreover, the regret bound is tighter as $O(C_{2,T}^*)$ always implies $O(C_{T}^*)$ when the minimizer sequence is uniformly bounded.


\begin{algorithm}[tb]
   \caption{Optimistic Online Newton (OON)}
   \label{alg:OON}
\begin{algorithmic}[1]
   \STATE {\bfseries Require:} Initial vector $\xb_1=\hat{\xb}_0\in \mathrm{R}^n$, 
  
   \FOR{$t=1,2,\ldots,T$}
   \STATE Play $\xb_t$ 
   \STATE Get the predicted second order information $\Mb_{t+1}$ and first order information $\mb_{t+1}$ of function $f_{t+1}$
   \STATE $\hat{\xb}_{t} = \hat{\xb}_{t-1} - \Hb_{t}^{-1}(\hat{\xb}_{t-1})\nabla f_{t}(\hat{\xb}_{t-1})$
   \STATE $\xb_{t+1} = \hat{\xb}_t - \Mb_{t+1}^{-1}(\hat{\xb}_t)\mb_{t+1}(\hat{\xb}_t)$
   \ENDFOR
\end{algorithmic}
\end{algorithm}

\subsection{Improved Dynamic Regret Bound for Optimistic Online Newton (OON)}\label{sec:newton}
The optimal bound on static regret for convex functions is $O(\sqrt{T})$ \citep{hazan2016introduction}. However, \citet{rakhlin2013optimization} show that by using gradient predictions, the regret bound can be $O(\sqrt{D_T})$, which is tighter if the predicted gradients are close enough to the actual gradients. Essentially, by using the predicted sequence, the learner aims at taking advantage of the niceness of the adversarial sequence (if possible).

In this section, we extend this idea to the case that the learner can use first and second order information simultaneously. The resulting algorithm, called optimistic online Newton (OON), is summarized in Algorithm \ref{alg:OON}. The learner performs a Newton update based on the predicted information but corrects this update using the true information. The following theorem presents the regret bound of OON in a local sense, which has the potential to significantly outperform first-order methods. The local nature of the result is not surprising as for the classical Newton method, quadratic convergence guarantee is only local (see e.g., Theorem 1.2.5 of \citep{nesterov1998introductory}).

\begin{theorem}\label{T:Online Newton dynamic bound}
Suppose that for any $t \in [T]$:
\begin{enumerate}
    \item The function $f_t:\mathrm{R}^n\to \mathrm{R}$ is $\mu$-strongly convex and $L$-smooth, and $\xb^*_t$ is bounded.
    \item $\exists L_H>0$ such that $\norm{\Hb_t(\xb) - \Hb_t(\xb_t^*)}\leq L_H\norm{\xb-\xb_t^*}$, where $\Hb_t(\xb) = \nabla^2 f_t(\xb)$.
    \item $\exists \xb_1 \in \mathrm{R}^n$ such that $\norm{\xb_1 - \xb_1^*}\leq \frac{\mu}{L_H}$.
    \item There exists a bound on local variations, where $\bar{c}\triangleq \max_{t\in\{2,\ldots,T\}}\norm{\xb_t^* - \xb_{t-1}^*} \leq \frac{\mu}{2L_H}.$
\end{enumerate}
Then, the dynamic regret for the sequence of actions $\xb_t$ by OON is bounded by
\begin{equation}\label{eq: Dynamic bound of OON.}
\begin{split}
    \mathbf{Reg}^d_T&\leq L\bigg( \frac{\norm{\hat{\xb}_1 - \xb_1^*}^2- \rho^{\prime}\norm{\hat{\xb}_T - \xb_T^*}^2}{1-\rho^{\prime}} + \frac{\rho^{\prime\prime}}{1-\rho^\prime}\sum_{t=2}^{T}\norm{\xb^*_{t}-\xb^*_{t-1}}^4\bigg)\\ &+ LD^\prime_T + L\norm{\Hb_{1}^{-1}(\hat{\xb}_{0})\nabla f_1(\hat{\xb}_{0})}^2,  
\end{split}
\end{equation}
where $\rho^{\prime}\triangleq\frac{1}{16}(1+c_1)^2(1+c_2),\quad \rho^{\prime \prime}\triangleq(\frac{L_H}{2\mu})^2(1+\frac{1}{c_1})^2(1+\frac{1}{c_2})$, $c_1$ and $c_2$ are any positive constants such that $0<\rho^{\prime}<1$, and $D^\prime_T\triangleq \sum_{t=2}^T \norm{\Mb_t^{-1}(\hat{\xb}_{t-1})\mb_t(\hat{\xb}_{t-1})-\Hb_{t}^{-1}(\hat{\xb}_{t-1})\nabla f_t(\hat{\xb}_{t-1})}^2$.
\end{theorem}
The theorem indicates that when the predicted information $(\Mb_t,\mb_t)$ is close to $(\Hb_t,\nabla f_t)$, $D'_T$ would be small and the dynamic regret bound would be close to $O(C_{4,T}^*)$. On the other hand, it can also be shown that if the learner uses stale information, i.e., $\Mb_t = \Hb_{t-1}$ and $\mb_t = \nabla f_{t-1}$, the regret bound of \eqref{eq: Dynamic bound of OON.} is $O(C_{2,T}^*)$, which matches the optimal worst-case bound. 

\begin{corollary}\label{Corollary: Optimistic online Newton stale information} Suppose that the assumptions of Theorem \ref{T:Online Newton dynamic bound} hold. If for $t=2,\ldots,T$, $(\Mb_t,\mb_t)=(\Hb_{t-1}, \nabla f_{t-1})$, the dynamic regret for the sequence of actions $\xb_t$ by OON is bounded by $O(C^*_{2,T})$.
\end{corollary}

\section{Dynamic Regret Bound for Online Multiple Gradient Descents (OMGD)}\label{sec:4}
In this section, we revisit the OMGD algorithm developed by \citep{zhang2017improved}, outlined in Algorithm \ref{alg:OMGD}. \citet{zhang2017improved} prove that by applying multiple gradient descents, the dynamic regret bound is $O(\min\{C^*_{T}, C^*_{2,T}\})$, which basically translates to $O(C^*_{2,T})$ as soon as the minimizers are uniformly bounded. We show that the regret bound of \citet{zhang2017improved} can be made more comprehensive by including $V_T$, the variation in the function sequence. The new bound, which takes the form of $O(\min\{V_T, C^*_{2,T}\})$, is presented below. 

\begin{algorithm}[tb]
   \caption{Online Multiple Gradient Descent (OMGD) \citep{zhang2017improved}}
   \label{alg:OMGD}
\begin{algorithmic}[1]
   \STATE {\bfseries Require:} Initial vector $\xb_1\in \mathrm{R}^n$, step size $\eta$, function parameters $\mu$ and $L$.
   \FOR{$t=1,2,\ldots,T$}
   \STATE Play $\xb_t$
   \STATE Receive the information of $f_t$
   \STATE Let $\zb_{t+1}^{(0)}=\xb_t$ and $K_t=\lceil \frac{-2\log(t)}{\log(1-\frac{2\eta\mu L}{\mu+L})} \rceil$
        \FOR{$j=1,2,\ldots,K_t$}
        \STATE $\zb_{t+1}^{(j)} = \zb_{t+1}^{(j-1)} - \eta\nabla f_t(\zb_{t+1}^{(j-1)})$  
        \ENDFOR
   \STATE $\xb_{t+1} = \zb_{t+1}^{(K_t)}$
   \ENDFOR
\end{algorithmic}
\end{algorithm}

\begin{theorem}\label{T:Online multiple gradient}
Suppose that for any $t\in [T]$, the function $f_t:\mathrm{R}^n\to \mathrm{R}$ is $\mu$-strongly convex and $L$-smooth, and $\xb^*_t$ is bounded.
For any $0<\eta\leq \frac{2}{\mu+L}$, the dynamic regret of OMGD is bounded as follows 
\begin{equation}\label{eq: Bound of Multiple GD}
    \mathbf{Reg}^d_T\leq \min \begin{cases} (f_1(\xb_1)-f_1(\xb^*_1)) + 2V_T + \frac{\pi^2 D^2L}{12}\\
    \frac{L}{2}\big[\norm{\xb_1 - \xb^*_1}^2 + 2D^2(\frac{\pi^2}{6}) + 2C^*_{2,T}\big]
    \end{cases},
\end{equation}
where $D\triangleq \max_{t\in[T]}\norm{\xb_t-\xb^*_t}$, and $V_T$ is defined with respect to $\Xc$ being the convex hull of $\{\xb_t,\xb^*_t\}_{t=1}^T$.
\end{theorem}

\subsection{Comparison of $C^*_{2,T}$ and $V_T$}
We now show that $V_T$ and $C^*_{2,T}$ are not directly comparable to each other. Therefore, having both of them present in the regret bound can only make the bound tighter. We construct two problem environments where $C^*_{2,T} \ll V_T$ and $V_T \ll C^*_{2,T} $, respectively.

Consider the following function sequence $f_t:\mathrm{R}^n\to \mathrm{R}$
\begin{eqnarray*}
    f_t(\xb) = \begin{cases}\norm{\xb-\xb^*}^2, \text{if } t \text{ is odd}\\
    \norm{\xb-\xb^*}^2 + 1, \text{if } t \text{ is even}
    \end{cases}
\end{eqnarray*}
For this function sequence, based on the regularity definitions \eqref{eq: path length order p} and \eqref{eq: variation in funciton values}, it is clear that
\begin{eqnarray*}
    &&C^*_{2,T} = \sum_{t=2}^T\norm{\xb^*_t-\xb^*_{t-1}}^2 = \sum_{t=2}^T\norm{\xb^*-\xb^*}^2 = 0.\\
    &&V_T = \sum_{t=2}^T\sup_{\xb\in\Xc}|f_t(\xb)-f_{t-1}(\xb)| = \Theta(T).
\end{eqnarray*}
In this case, we see that $V_T$ is much larger than $C^*_{2,T}$. On the other hand, consider another function sequence $f_t:\mathrm{R}^n\to \mathrm{R}$
\begin{eqnarray*}
    f_t(\xb) = \begin{cases}\frac{\norm{\xb}^2}{t}, \text{if } t \text{ is odd}\\
    \frac{\norm{\xb-\yb}^2}{t}, \text{if } t \text{ is even}
    \end{cases}
\end{eqnarray*}
We have that
\begin{eqnarray*}
    &&C^*_{2,T} = \sum_{t=2}^T\norm{\xb^*_t-\xb^*_{t-1}}^2 = \sum_{t=2}^T\norm{\yb}^2 = \Theta(T).\\
    &&V_T = \sum_{t=2}^T\sup_{\xb\in\Xc}|f_t(\xb)-f_{t-1}(\xb)| = \sum_{t=2}^T\sup_{\xb\in\Xc}\bigg|\frac{2\xb^\top \yb + \yb^\top \yb}{t}\bigg|\leq O(\log T).
\end{eqnarray*}
We can see that $V_T$ is considerably smaller than $C^*_{2,T}$ in this scenario. With these two examples, it can be seen that if the regret bound only uses one regularity, it is possible that the resulting bound is not tight.

\section{Discussion for Constrained Setups}
In this section, we present that for OPGD and OMGD, if the function domain is assumed to be constrained, similar theoretical results can be achieved. From the aspect of the learning process, since the domain set is constrained, the projection step needs to be added (see algorithms \ref{alg:OPGD(constrained)} \& \ref{alg:OMGD(constrained)}). Note that $\Pi_{\Xc}(\yb) = \argmin_{\xb\in\Xc}\norm{\yb-\xb}$ and $\Pi_{\Xc}^{\Ab}(\yb) = \argmin_{\xb\in\Xc}\norm{\yb-\xb}_{\Ab}$, where $\norm{\xb}_{\Ab} = \sqrt{\xb^\top \Ab \xb}$ for a positive-definite matrix $\Ab$. 

\begin{algorithm}[ht]
   \caption{Online Preconditioned Gradient Descent (OPGD) for Constrained Setups}
   \label{alg:OPGD(constrained)}
\begin{algorithmic}[1]
   \STATE {\bfseries Require:} Initial vector $\xb_1\in \Xc$, step size $\eta$, a sequence of positive definite matrices $\Ab_t$
  
   \FOR{$t=1,2,\ldots,T$}
   \STATE Play $\xb_t$
   \STATE Observe the gradient of the current action $\nabla f_t(\xb_t)$
   \STATE $\xb_{t+1} = \Pi_{\Xc}^{\Ab_t}\bigg(\xb_t - \eta\Ab^{-1}_t\nabla f_t(\xb_t)\bigg)$
   \ENDFOR
\end{algorithmic}
\end{algorithm}

\begin{theorem}\label{T: Improved OPGD dynamic bound (constrained)} (The constrained version of Theorem \ref{T: Improved OPGD dynamic bound}) Suppose that for any $t\in [T]$:
\begin{enumerate}
    \item The function $f_t:\Xc\to \mathrm{R}$ is $\mu$-strongly convex and $L$-smooth over $\Xc$.
    \item The preconditioning matrix $\Ab_t$ satisfies $\lambda^\prime\cdot \Ib\preceq \Ab_t \preceq \lambda \cdot \Ib.$
    \item The condition number satisfies $\frac{\lambda}{\lambda\prime} < 1+\frac{\mu^2}{4L^2}.$
\end{enumerate}
If we set $\eta=\frac{\lambda^\prime\mu}{2L^2}$, the dynamic regret for the sequence of actions $\xb_t$ generated by OPGD is bounded as follows
\begin{eqnarray*}
\mathbf{Reg}^d_T&\leq& \left(\frac{L^2}{\mu}\right)\left(\frac{4L^2\lambda-\mu^2\lambda^\prime}{\mu^2\lambda^\prime - 4L^2\lambda + 4L^2\lambda^\prime}\right)\sum_{t=2}^{T+1} \norm{\xb_t^* - \xb_{t-1}^*}^2 + \left(\frac{L^2\lambda}{\lambda^\prime\mu} - \frac{\mu}{4}\right)\norm{\xb_1 - \xb_1^*}^2\\&+&\frac{\mu D}{2L}\sum_{t=1}^T\norm{\nabla f_t(\xb^*_t)},
\end{eqnarray*}
where $D=\max_{\xb,\yb\in\Xc}\norm{\xb-\yb}$.
\end{theorem}

\begin{theorem}\label{T:Online multiple gradient (constrained)}
Suppose that for any $t\in [T]$, the function $f_t:\Xc\to \mathrm{R}$ is $\mu$-strongly convex and $L$-smooth.
For any $0<\eta\leq \frac{1}{L}$, the dynamic regret of OMGD is bounded as follows 
\begin{equation}\label{eq: Bound of Multiple GD (constrained)}
    \mathbf{Reg}^d_T\leq \min \begin{cases} (f_1(\xb_1)-f_1(\xb^*_1)) + 2V_T + \frac{\pi^2 D^2L}{12} + D\sum_{t=2}^T\norm{\nabla f_{t-1}(\xb^*_{t-1})}\\
    \frac{L}{2}\big[\norm{\xb_1 - \xb^*_1}^2 + 2D^2(\frac{\pi^2}{6}) + 2C^*_{2,T}\big]+ D\sum_{t=1}^T\norm{\nabla f_t (\xb^*_t)}
    \end{cases},
\end{equation}
where $D\triangleq \max_{\xb,\yb\in\Xc}\norm{\xb-\yb}$.
\end{theorem}

For Theorems \ref{T: Improved OPGD dynamic bound (constrained)} \& \ref{T:Online multiple gradient (constrained)}, if $\xb^*_t$ is located in the relative interior of $\Xc$ (i.e., $\nabla f_t(\xb^*_t)=0$) for any $t\in [T]$, the theoretical results of the unconstrained case can be recovered.

\begin{algorithm}[ht]
   \caption{Online Multiple Gradient Descent (OMGD) \citep{zhang2017improved} for Constrained Setups}
   \label{alg:OMGD(constrained)}
\begin{algorithmic}[1]
   \STATE {\bfseries Require:} Initial vector $\xb_1\in \Xc$, step size $\eta$, function parameters $\mu$ and $L$.
   \FOR{$t=1,2,\ldots,T$}
   \STATE Play $\xb_t$
   \STATE Receive the information of $f_t$
   \STATE Let $\zb_{t+1}^{(0)}=\xb_t$ and $K_t=\lceil \frac{-2\log(t)}{\log(1-\frac{2\mu}{1/\eta+\mu})} \rceil$
        \FOR{$j=1,2,\ldots,K_t$}
        \STATE $\zb_{t+1}^{(j)} = \Pi_{\Xc}\bigg(\zb_{t+1}^{(j-1)} - \eta\nabla f_t(\zb_{t+1}^{(j-1)})\bigg)$  
        \ENDFOR
   \STATE $\xb_{t+1} = \zb_{t+1}^{(K_t)}$
   \ENDFOR
\end{algorithmic}
\end{algorithm}

\section{Concluding Remarks}
In this paper, we analyzed the dynamic regret for unconstrained online optimization, where the function sequence is strongly convex and smooth. We first proposed online preconditioned gradient descent (OPGD), which achieves the optimal regret bound of $O(C^*_{2,T})$ with one gradient query per round. Next, we developed optimistic online Newton (OON) method, which uses predictions of Hessians and gradients in the update process. We proved a (local) dynamic regret bound scaling as $O(D^\prime_T+C^*_{4,T})$, where $D^\prime_T$ measures the dissimilarity between the predicted information and the true information. If $D^\prime_T$ is small, this algorithm provides an improvement over the regret rate of $O(C^*_{2,T})$. We further verified that a conservative learner, that uses stale (previous round) information, always incurs a regret that is no worse than $O(C^*_{2,T})$, recovering the optimal worst-case. We finally revisited the online multiple gradient descent (OMGD) algorithm of \citep{zhang2017improved} and provided complementary analysis that shows the regret rate of $O(\min\{V_T, C^*_{2,T}\})$ for OMGD. The main benefit of the bound is that $V_T$ and $C_{2,T}^*$ are not comparable, and including $V_T$ in the regret bound can only make it tighter. 

Our results mainly provide the following insight about dynamic online learning. It is unlikely that the smoothness assumption alone can provide improved regret rates over convex setting. This is in contrast to classical optimization that deals with {\it time-invariant} functions, where smoothness does provide faster convergence rates for optimization algorithms (e.g., gradient descent). In dynamic online setting, to decide on action $\xb_t$, the learner can only use the information of $\{f_s\}_{s=1}^{t-1}$; however, the consequence of this action is evaluated at $f_t$. Hence, the adversary can always change the function sequence drastically so that the learner suffers a large loss. Therefore, the assumption of strong convexity comes to rescue by translating the closeness of objective values to the closeness of actions. As a result, with smoothness and strong convexity together, dynamic regret bounds can be considerably better than the convex case. Given the improvement of OON over OGD when the predictable sequence is accurate enough, it would be interesting to see whether we can leverage predicted higher-order information to further improve the dynamic regret bound.


\newpage

\section{Appendix}

\begin{proof}[\textbf{Proof of Theorem \ref{T: Improved OPGD dynamic bound}}]
Recall the update rule of Algorithm \ref{alg:OPGD}: $\xb_{t+1} = \xb_t - \eta\Ab_t^{-1}\nabla f_t(\xb_t)$. Subtracting $\xb_t^*$ on both sides and computing the norm induced by $\Ab_t$, we get
\begin{equation}\label{eq:OPGD proof eq1}
\begin{gathered}
    (\xb_{t+1} - \xb_t^*)^\top\Ab_t(\xb_{t+1} - \xb_t^*) = \\ (\xb_{t} - \xb_t^*)^\top\Ab_t(\xb_{t} - \xb_t^*) - 2\eta\nabla f_t(\xb_t)^\top(\xb_{t} - \xb_t^*) + \eta^2\nabla f_t(\xb_t)^\top\Ab_t^{-1}\nabla f_t(\xb_t). 
\end{gathered}
\end{equation}
Dividing both sides by $2\eta$ and rearranging the equation, we get
\begin{equation}\label{eq:OPGD proof eq2}
\begin{gathered}
    \nabla f_t(\xb_t)^\top(\xb_{t} - \xb_t^*) = \\
     \frac{\eta}{2}\nabla f_t(\xb_t)^\top\Ab_t^{-1}\nabla f_t(\xb_t) + \frac{1}{2\eta}(\xb_{t} - \xb_t^*)^\top\Ab_t(\xb_{t} - \xb_t^*) - \frac{1}{2\eta}(\xb_{t+1} - \xb_t^*)^\top\Ab_t(\xb_{t+1} - \xb_t^*).
\end{gathered}
\end{equation}

Due to $\lambda^\prime\cdot \Ib \preceq \Ab_t \preceq \lambda\cdot\Ib$, we have
\begin{equation}\label{eq:OPGD proof eq3}
    \frac{1}{2\eta}(\xb_{t} - \xb_t^*)^\top\Ab_t(\xb_{t} - \xb_t^*)\leq \frac{\lambda}{2\eta} \norm{\xb_{t} - \xb_t^*}^2. 
\end{equation}
On the other hand, 
\begin{eqnarray*}
    \norm{\xb_{t+1} - \xb^*_t}^2_{\Ab_t}&\geq&\lambda^\prime\norm{\xb_{t+1}-\xb^*_t}^2 \nonumber\\
    &\geq&\lambda^\prime(\norm{\xb_{t+1} - \xb^*_{t+1}}-\norm{\xb^*_{t+1} - \xb^*_{t}})^2 \nonumber\\
    &=&\lambda^\prime\bigg(\norm{\xb_{t+1} - \xb^*_{t+1}}^2-2\norm{\xb_{t+1} - \xb^*_{t+1}}\norm{\xb^*_{t+1} - \xb^*_{t}} + \norm{\xb^*_{t+1} - \xb^*_{t}}^2\bigg)\nonumber\\
    &\geq& \lambda^\prime\bigg[(1-c)\norm{\xb_{t+1} - \xb^*_{t+1}}^2 + (1-\frac{1}{c})\norm{\xb^*_{t+1} - \xb^*_{t}}^2\bigg],
\end{eqnarray*}
where the second inequality comes from the triangle inequality, and the third inequality comes from the fact that $2ab\leq ca^2 + \frac{1}{c}b^2,\quad \forall c>0$. Based on the result above, we have
\begin{equation}\label{eq:OPGD proof eq4}
    - \frac{1}{2\eta}(\xb_{t+1} - \xb_t^*)^\top\Ab_t(\xb_{t+1} - \xb_t^*)\leq - \frac{\lambda^\prime}{2\eta}\bigg[(1-c)\norm{\xb_{t+1} - \xb^*_{t+1}}^2 + (1-\frac{1}{c})\norm{\xb^*_{t+1} - \xb^*_{t}}^2\bigg].
\end{equation}\\
Combining \eqref{eq:OPGD proof eq2}, \eqref{eq:OPGD proof eq3} and \eqref{eq:OPGD proof eq4}, we get
\begin{equation}\label{eq:OPGD proof eq5}
\begin{split}
    \nabla f_t(\xb_t)^\top(\xb_{t} - \xb_t^*)&\leq \frac{\eta}{2}\nabla f_t(\xb_t)^\top\Ab_t^{-1}\nabla f_t(\xb_t) + \frac{\lambda}{2\eta} \norm{\xb_{t} - \xb_t^*}^2\\ &- \frac{\lambda^\prime}{2\eta}\bigg[(1-c)\norm{\xb_{t+1} - \xb^*_{t+1}}^2 + (1-\frac{1}{c})\norm{\xb^*_{t+1} - \xb^*_{t}}^2\bigg].
\end{split}
\end{equation}
Summing \eqref{eq:OPGD proof eq5} over $t\in [T]$, we get
\begin{equation}\label{eq:OPGD proof eq6}
\begin{split}
    \sum_{t=1}^T\nabla f_t(\xb_t)^\top (\xb_t - \xb_t^*)&\leq \frac{\eta}{2}\sum_{t=1}^T \nabla f_t(\xb_t)^\top\Ab_t^{-1}\nabla f_t(\xb_t) + \frac{1}{2\eta}(\lambda-\lambda^\prime(1-c))\sum^T_{t=1}\norm{\xb_t - \xb_t^*}^2\\ &+ \frac{\lambda^\prime}{2\eta}(\frac{1}{c}-1)\sum^T_{t=1}\norm{\xb_{t+1}^* - \xb_t^*}^2 + \frac{\lambda^\prime(1-c)}{2\eta}\norm{\xb_1-\xb_1^*}^2\\&-\frac{\lambda^\prime(1-c)}{2\eta}\norm{\xb_{T+1} - \xb_{T+1}^*}^2.
\end{split}
\end{equation}
By making $c\in (0,1)$, the term $\frac{\lambda^\prime(1-c)}{2\eta}\norm{\xb_{T+1} - \xb_{T+1}^*}^2$ can be dropped from \eqref{eq:OPGD proof eq6}. 

As the functions are assumed to be unconstrained, the gradient at the minimizer is zero, and\\ $\nabla f_t(\xb_t)^\top\Ab_t^{-1}\nabla f_t(\xb_t)$ can be bounded using $L$-smoothness as follows
\begin{eqnarray*}
    \nabla f_t(\xb_t)^\top\Ab_t^{-1}\nabla f_t(\xb_t) &=& (\nabla f_t(\xb_t)-\nabla f_t(\xb_t^*))^\top\Ab_t^{-1}(\nabla f_t(\xb_t)-\nabla f_t(\xb_t^*))\\
    &\leq&\frac{L^2}{\lambda^\prime}\norm{\xb_t-\xb_t^*}^2.
\end{eqnarray*}
Substituting above into \eqref{eq:OPGD proof eq6} and rearranging it, we get
\begin{equation}\label{eq:OPGD proof eq7}
\begin{split}
    &\sum_{t=1}^T\nabla f_t(\xb_t)^\top (\xb_t - \xb_t^*)
    -\bigg[\frac{\lambda-\lambda^\prime(1-c)}{2\eta}+\frac{\eta L^2}{2\lambda^\prime}\bigg]\sum_{t=1}^T\norm{\xb_t-\xb_t^*}^2\\\leq &\frac{\lambda^\prime}{2\eta}(\frac{1}{c}-1)\sum^T_{t=1}\norm{\xb_{t+1}^* - \xb_t^*}^2 + \frac{\lambda^\prime(1-c)}{2\eta}\norm{\xb_1-\xb_1^*}^2.
\end{split}
\end{equation}
Based on the assumption of strong convexity, we have the following relationship
$$\mathbf{Reg}^d_T=\sum_{t=1}^Tf_t(\xb_t)-f_t(\xb_t^*)\leq\sum_{t=1}^T\nabla f_t(\xb_t)^\top (\xb_t-\xb_t^*)-\frac{\mu}{2}\norm{\xb_t-\xb_t^*}^2.$$
Therefore, in view of \eqref{eq:OPGD proof eq7}, if
\begin{equation}\label{eq:OPGD proof eq8}
    \bigg[\frac{\lambda-\lambda^\prime(1-c)}{\eta}+\frac{\eta L^2}{\lambda^\prime}\bigg]\leq \mu,
\end{equation}
then $\mathbf{Reg}^d_T\leq \frac{\lambda^\prime}{2\eta}(\frac{1}{c}-1)\sum^T_{t=1}\norm{\xb_{t+1}^* - \xb_t^*}^2 + \frac{\lambda^\prime(1-c)}{2\eta}\norm{\xb_1-\xb_1^*}^2$.\\\\
Optimizing the left hand side of \eqref{eq:OPGD proof eq8} over $\eta$, we have the following relation
$$
\sqrt{\lambda-\lambda^\prime(1-c)} \leq \frac{\mu}{2L}\sqrt{\lambda^\prime}
$$
for $\eta=\frac{\sqrt{\lambda\lambda^\prime-{\lambda^\prime}^2(1-c)}}{L}$.
By setting $c = \frac{\mu^2}{4L^2}-\frac{\lambda}{\lambda^\prime}+1$, the inequality above holds, and $\eta = \frac{\mu\lambda^\prime}{2L^2}$. Note that the choice of $c$ satisfies the constraint $c\in (0,1)$ as $\frac{\lambda}{\lambda^\prime}<\frac{\mu^2}{4L^2}+1$ by assumption.

By substituting expressions of $c$ and $\eta$ into \eqref{eq:OPGD proof eq7}, the theorem is proved.
\end{proof}

In order to prove Theorem \ref{T:Online Newton dynamic bound}, we need the following two lemmas.\\

\begin{lemma}\label{L:Online Newton quadratic convergence}
Suppose the assumptions of Theorem \ref{T:Online Newton dynamic bound} hold. Then, we have the following inequality
\begin{equation}\label{eq: quadratic convergence}
\norm{\hat{\xb}_{t+1} - \xb_{t+1}^*}\leq \frac{L_H}{2\mu}\norm{\hat{\xb}_{t} - \xb_{t+1}^*}^2,    
\end{equation}
for $t=0,\ldots,T$.
\begin{proof}
The proof closely follows the local quadratic convergence of Newton method. Consider step 5 of Algorithm \ref{alg:OON}. By subtracting $\xb_{t+1}^*$ on both sides, we get
    \begin{eqnarray*}
        \hat{\xb}_{t+1} - \xb_{t+1}^* &=& \hat{\xb}_{t} - \xb_{t+1}^* - \Hb_{t+1}^{-1}(\hat{\xb}_{t})\nabla f_{t+1}(\hat{\xb}_{t})\\
        &=& \hat{\xb}_{t} - \xb_{t+1}^* + \Hb_{t+1}^{-1}(\hat{\xb}_{t})(\nabla f_{t+1}(\xb_{t+1}^*)-\nabla f_{t+1}(\hat{\xb}_{t})).\\
    \end{eqnarray*}
    The equation above can be written as 
    \begin{eqnarray*}
        \hat{\xb}_{t+1} - \xb_{t+1}^* &=& \hat{\xb}_{t} - \xb_{t+1}^* + \Hb_{t+1}^{-1}(\hat{\xb}_{t})\int_0^1 \Hb_{t+1}(\hat{\xb}_{t} + \tau (\xb^*_{t+1}-\hat{\xb}_{t}))(\xb^*_{t+1}-\hat{\xb}_{t})d\tau\\
        &=& \Hb_{t+1}^{-1}(\hat{\xb}_{t})\Hb_{t+1}(\hat{\xb}_{t})(\hat{\xb}_{t} - \xb_{t+1}^*)\\ 
        &+& \Hb_{t+1}^{-1}(\hat{\xb}_{t})\int_0^1 \Hb_{t+1}(\hat{\xb}_{t} + \tau (\xb^*_{t+1}-\hat{\xb}_{t}))(\xb^*_{t+1}-\hat{\xb}_{t})d\tau\\
        &=& \Hb_{t+1}^{-1}(\hat{\xb}_{t})\int_0^1 \Big[\Hb_{t+1}(\hat{\xb}_{t} + \tau (\xb^*_{t+1}-\hat{\xb}_{t})) - \Hb_{t+1}(\hat{\xb}_{t})\Big]\Big(\xb^*_{t+1}-\hat{\xb}_{t}\Big)d\tau.
    \end{eqnarray*}
    By computing the norms of both sides, we get
    \begin{eqnarray*}
        \norm{\hat{\xb}_{t+1} - \xb_{t+1}^*} &=& \norm{\Hb_{t+1}^{-1}(\hat{\xb}_{t})\int_0^1 \Big[\Hb_{t+1}(\hat{\xb}_{t} + \tau (\xb^*_{t+1}-\hat{\xb}_{t})) - \Hb_{t+1}(\hat{\xb}_{t})\Big]\Big(\xb^*_{t+1}-\hat{\xb}_{t}\Big)d\tau}\\
        &\leq& \norm{\Hb_{t+1}^{-1}(\hat{\xb}_{t})} \int_0^1 \norm{\Hb_{t+1}(\hat{\xb}_{t} + \tau (\xb^*_{t+1}-\hat{\xb}_{t})) - \Hb_{t+1}(\hat{\xb}_{t})}\norm{\xb^*_{t+1}-\hat{\xb}_{t}}d\tau\\
        &\leq& \norm{\Hb_{t+1}^{-1}(\hat{\xb}_{t})} \int_0^1 \tau L_H\norm{\hat{\xb}_{t}-\xb_{t+1}^*}^2d\tau\\
        &\leq& \frac{L_H}{2\mu}\norm{\hat{\xb}_{t}-\xb_{t+1}^*}^2,
    \end{eqnarray*}
    where the second inequality comes from the Lipschitz continuity in Hessian and the third inequality comes from strong convexity.
\end{proof}
\end{lemma}

The following result is close to Lemma 3 of \citep{lesage2020second}, but since our update is slightly different, we provide the proof for completeness. 

\begin{lemma}\label{L:Boundness of distance to the minimizer}
Suppose the assumptions of Theorem \ref{T:Online Newton dynamic bound} holds. Then, we have the following inequality
\begin{equation}\label{eq: boundness}
    \norm{\hat{\xb}_{t} - \xb_{t+1}^*}\leq \frac{\mu}{L_H},
\end{equation}
for $t=0,\ldots,T$.
\begin{proof} 
The proof is by induction. When $t=0$, we have $\norm{\hat{\xb}_{0} - \xb_{1}^*} = \norm{\xb_1-\xb_1^*}\leq \frac{\mu}{L_H}$. Suppose $\norm{\hat{\xb}_{t} - \xb_{t+1}^*}\leq \frac{\mu}{L_H}$ holds. Now for the case $t+1$, we have the following relation
    \begin{eqnarray*}
        \norm{\hat{\xb}_{t+1} - \xb_{t+2}^*} &\leq& \norm{\hat{\xb}_{t+1} - \xb_{t+1}^*} + \norm{\xb_{t+1}^* - \xb_{t+2}^*}\\
        &\leq& \frac{L_H}{2\mu}\norm{\hat{\xb}_{t} - \xb_{t+1}^*}^2 + \bar{c}\\
        &\leq& \frac{L_H}{2\mu}\left(\frac{\mu}{L_H}\right)^2 + \frac{\mu}{2L_H}\\
        &=&\frac{\mu}{L_H},
    \end{eqnarray*}
    where the second inequality comes from Lemma \ref{L:Online Newton quadratic convergence}. 
\end{proof}
\end{lemma}

\begin{proof}[\textbf{Proof of Theorem \ref{T:Online Newton dynamic bound}}]
First, we bound the term $\norm{\hat{\xb}_{t+1} - \xb_{t+1}^*}$ as follows
\begin{eqnarray*}
    \norm{\hat{\xb}_{t+1} - \xb_{t+1}^*} &\leq& \frac{L_H}{2\mu}\norm{\hat{\xb}_{t} - \xb_{t+1}^*}^2\\
    &\leq& \frac{L_H}{2\mu} \bigg[(1+c_1)\norm{\hat{\xb}_{t} - \xb_{t}^*}^2 + \left(1+\frac{1}{c_1}\right)\norm{\xb_{t}^* - \xb_{t+1}^*}^2\bigg]\\
    &\leq& \frac{1}{4}(1+c_1)\norm{\hat{\xb}_{t} - \xb_{t}^*} + \frac{L_H}{2\mu}\left(1+\frac{1}{c_1}\right)\norm{\xb_{t}^* - \xb_{t+1}^*}^2.
\end{eqnarray*}
The first inequality is based on equation \eqref{eq: quadratic convergence}; the second inequality is based on the fact that $(a+b)^2\leq (1+c)a^2 + (1+\frac{1}{c})b^2,\quad \forall c>0$, and the third inequality is based on the fact that combining \eqref{eq: quadratic convergence} and \eqref{eq: boundness}, we can get $\norm{\hat{\xb}_t - \xb^*_t}\leq \frac{\mu}{2L_H}$.\\\\
Applying the inequality $(a+b)^2\leq (1+c)a^2 + (1+\frac{1}{c})b^2,\quad \forall c>0$ to the above inequality again, we can find {\small
\begin{equation}\label{quartic bound 1}
    \begin{split}
    \norm{\hat{\xb}_{t+1} - \xb_{t+1}^*}^2 &\leq \bigg[ \frac{1}{4}(1+c_1)\norm{\hat{\xb}_{t} - \xb_{t}^*} + \frac{L_H}{2\mu}\left(1+\frac{1}{c_1}\right)\norm{\xb_{t}^* - \xb_{t+1}^*}^2\bigg]^2\\
    &\leq \frac{1}{16}(1+c_1)^2(1+c_2)\norm{\hat{\xb}_{t} - \xb_{t}^*}^2 + \left(\frac{L_H}{2\mu}\right)^2\left(1+\frac{1}{c_1}\right)^2\left(1+\frac{1}{c_2}\right)\norm{\xb_{t}^* - \xb_{t+1}^*}^4.
    \end{split}
\end{equation}}
Let $\rho^\prime\triangleq\frac{1}{16}(1+c_1)^2(1+c_2)$ and $\rho^{\prime\prime}\triangleq(\frac{L_H}{2\mu})^2(1+\frac{1}{c_1})^2(1+\frac{1}{c_2})$.\\\\
Summing the term $\norm{\hat{\xb}_{t+1} - \xb_{t+1}^*}^2$ over $t$, based on \eqref{quartic bound 1}, we get
\begin{eqnarray*}
    \sum_{t=1}^T\norm{\hat{\xb}_t - \xb_t^*}^2&\leq& \norm{\hat{\xb}_1 - \xb_1^*}^2 + \rho^\prime\sum_{t=1}^{T-1}\norm{\hat{\xb}_t - \xb_t^*}^2 + \rho^{\prime\prime}\sum_{t=2}^{T}\norm{\xb_{t}^* - \xb_{t-1}^*}^4.
\end{eqnarray*}
Adding and subtracting $\rho^\prime\norm{\hat{\xb}_T - \xb_T}^2$ to the right hand side, we get
\begin{equation}\label{eq: quartic bound 2}
    \sum_{t=1}^T\norm{\hat{\xb}_t - \xb_t^*}^2\leq \norm{\hat{\xb}_1 - \xb_1^*}^2 -\rho^\prime\norm{\hat{\xb}_T - \xb_T}^2 + \rho^\prime\sum_{t=1}^{T}\norm{\hat{\xb}_t - \xb_t^*}^2 + \rho^{\prime\prime}\sum_{t=2}^{T}\norm{\xb_{t}^* - \xb_{t-1}^*}^4.
\end{equation}
Choosing any $c_1,c_2>0$ such that $\rho^\prime<1$, we can regroup \eqref{eq: quartic bound 2} to derive
\begin{equation}\label{eq: quartic bound 3}
    \sum_{t=1}^T\norm{\hat{\xb}_t - \xb_t^*}^2\leq \frac{\norm{\hat{\xb}_1 - \xb_1^*}^2 -\rho^\prime\norm{\hat{\xb}_T - \xb_T}^2}{1-\rho\prime} + \frac{\rho^{\prime\prime}}{1-\rho^\prime}\sum_{t=2}^{T}\norm{\xb_{t}^* - \xb_{t-1}^*}^4.
\end{equation}
Now let us return to the loss difference $f_t(\xb_t) - f_t(\xb_t^*)$. Based on the assumption of smoothness, we have
\begin{eqnarray*}
f_t(\xb_t) - f_t(\xb_t^*) \leq \frac{L}{2}\norm{\xb_t - \xb_t^*}^2.
\end{eqnarray*}
Summing the above inequality over $t\in [T]$, we get
\begin{equation}\label{eq: quartic bound 4}
    \begin{split}
    \mathbf{Reg}^d_T &= \sum_{t=1}^T f_t(\xb_t) - f_t(\xb_t^*)\\
    &\leq \frac{L}{2}\sum_{t=1}^T\norm{\xb_t - \xb_t^*}^2\\
    &\leq \frac{L}{2}\sum_{t=1}^T (2\norm{\xb_t-\hat{\xb}_t}^2 + 2\norm{\hat{\xb}_t - \xb_t^*}^2)\\
    &\leq L\sum_{t=1}^T\norm{\xb_t-\hat{\xb}_t}^2 + L\bigg(\frac{\norm{\hat{\xb}_1 - \xb_1^*}^2 -\rho^\prime\norm{\hat{\xb}_T - \xb_T}^2}{1-\rho\prime} + \frac{\rho^{\prime\prime}}{1-\rho^\prime}\sum_{t=2}^{T}\norm{\xb_{t}^* - \xb_{t-1}^*}^4\bigg)\\
    &=L\bigg(\frac{\norm{\hat{\xb}_1 - \xb_1^*}^2 -\rho^\prime\norm{\hat{\xb}_T - \xb_T}^2}{1-\rho\prime} + \frac{\rho^{\prime\prime}}{1-\rho^\prime}\sum_{t=2}^{T}\norm{\xb_{t}^* - \xb_{t-1}^*}^4\bigg)\\
    &+ L\norm{\Hb_{1}^{-1}(\hat{\xb}_{0})\nabla f_1(\hat{\xb}_{0})}^2 +L\sum_{t=2}^T \norm{\Mb_t^{-1}(\hat{\xb}_{t-1})\mb_t(\hat{\xb}_{t-1})-\Hb_{t}^{-1}(\hat{\xb}_{t-1})\nabla f_t(\hat{\xb}_{t-1})}^2,
    \end{split}
\end{equation}
where we used \eqref{eq: quartic bound 3} as well as steps 5-6 of Algorithm \ref{alg:OON}.  
\end{proof}
\begin{proof}[\textbf{Proof of Corollary \ref{Corollary: Optimistic online Newton stale information}}]
Based on \eqref{eq: quartic bound 4}, we have that
\begin{equation}\label{eq: quartic bound 5}
\begin{split}
    \mathbf{Reg}^d_T &\leq L\bigg(\frac{\norm{\hat{\xb}_1 - \xb_1^*}^2 -\rho^\prime\norm{\hat{\xb}_T - \xb_T}^2}{1-\rho\prime} + \frac{\rho^{\prime\prime}}{1-\rho^\prime}\sum_{t=2}^{T}\norm{\xb_{t}^* - \xb_{t-1}^*}^4\bigg)\\
    &+ L\norm{\Hb_{1}^{-1}(\hat{\xb}_{0})\nabla f_1(\hat{\xb}_{0})}^2 +L\sum_{t=2}^T \norm{\Mb_t^{-1}(\hat{\xb}_{t-1})\mb_t(\hat{\xb}_{t-1})-\Hb_{t}^{-1}(\hat{\xb}_{t-1})\nabla f_t(\hat{\xb}_{t-1})}^2.
\end{split}
\end{equation}
For last term, by replacing $(\Mb_t,\mb_t)$ with $(\Hb_{t-1}, \nabla f_{t-1})$, we obtain 
\begin{equation}\label{eq: quartic bound 6}
\begin{split}
    &\norm{\Mb_t^{-1}(\hat{\xb}_{t-1})\mb_t(\hat{\xb}_{t-1})-\Hb_{t}^{-1}(\hat{\xb}_{t-1})\nabla f_t(\hat{\xb}_{t-1})}^2\\ = &\norm{\Hb_{t-1}^{-1}(\hat{\xb}_{t-1})\nabla f_{t-1}(\hat{\xb}_{t-1})-\Hb_{t}^{-1}(\hat{\xb}_{t-1})\nabla f_t(\hat{\xb}_{t-1})}^2\\
    =&\norm{\Hb_{t-1}^{-1}(\hat{\xb}_{t-1})\big(\nabla f_{t-1}(\hat{\xb}_{t-1})-\nabla f_{t-1}(\xb^*_{t-1})\big)-\Hb_{t}^{-1}(\hat{\xb}_{t-1})\big(\nabla f_t(\hat{\xb}_{t-1})-\nabla f_t(\xb^*_{t})\big)}^2\\
    \leq &2\norm{\Hb_{t-1}^{-1}(\hat{\xb}_{t-1})\big(\nabla f_{t-1}(\hat{\xb}_{t-1})-\nabla f_{t-1}(\xb^*_{t-1})\big)}^2 + 2\norm{\Hb_{t}^{-1}(\hat{\xb}_{t-1})\big(\nabla f_t(\hat{\xb}_{t-1})-\nabla f_t(\xb^*_{t})\big)}^2\\
    \leq &\frac{2L^2}{\mu^2}\big(\norm{\hat{\xb}_{t-1} - \xb^*_{t-1}}^2+\norm{\hat{\xb}_{t-1} - \xb^*_{t}}^2\big)\\
    \leq &\frac{2L^2}{\mu^2}\norm{\hat{\xb}_{t-1} - \xb^*_{t-1}}^2 + \frac{4L^2}{\mu^2}\norm{\hat{\xb}_{t-1} - \xb^*_{t-1}}^2 + \frac{4L^2}{\mu^2}\norm{\xb^*_t - \xb^*_{t-1}}^2\\
    =&\frac{6L^2}{\mu^2}\norm{\hat{\xb}_{t-1} - \xb^*_{t-1}}^2 + \frac{4L^2}{\mu^2}\norm{\xb^*_t - \xb^*_{t-1}}^2,
\end{split}
\end{equation}
where the second inequality comes from strong convexity and smoothness.\\\\
Summing \eqref{eq: quartic bound 6} over $t=2,\ldots,T$, we get
\begin{equation}\label{eq: quartic bound 7}
\begin{split}
    &\sum_{t=2}^T\norm{\Mb_t^{-1}(\hat{\xb}_{t-1})\mb_t(\hat{\xb}_{t-1})-\Hb_{t}^{-1}(\hat{\xb}_{t-1})\nabla f_t(\hat{\xb}_{t-1})}^2\\
    \leq &\frac{6L^2}{\mu^2}\sum_{t=1}^{T-1}\norm{\hat{\xb}_t-\xb^*_t}^2 + \frac{4L^2}{\mu^2}\sum_{t=2}^T\norm{\xb^*_t-\xb^*_{t-1}}^2\\
    \leq &\frac{6L^2}{\mu^2}\bigg(\frac{\norm{\hat{\xb}_1 - \xb_1^*}^2 -\rho^\prime\norm{\hat{\xb}_T - \xb_T}^2}{1-\rho\prime} + \frac{\rho^{\prime\prime}}{1-\rho^\prime}\sum_{t=2}^{T}\norm{\xb_{t}^* - \xb_{t-1}^*}^4\bigg) + \frac{4L^2}{\mu^2}\sum_{t=2}^T\norm{\xb^*_t-\xb^*_{t-1}}^2,
\end{split}    
\end{equation}
where the second inequality comes from \eqref{eq: quartic bound 3}.\\\\
Substituting \eqref{eq: quartic bound 7} into \eqref{eq: quartic bound 5}, we have 
\begin{equation}\label{eq: quartic bound 8}
\begin{split}
    \mathbf{Reg}^d_T &\leq \frac{6L^3+\mu^2  L}{\mu^2}\bigg(\frac{\norm{\hat{\xb}_1 - \xb^*_1}^2}{1-\rho^\prime} + \frac{\rho^{\prime\prime}}{1-\rho^\prime}\sum_{t=2}^{T}\norm{\xb_{t}^* - \xb_{t-1}^*}^4\bigg) + L\norm{\Hb_{1}^{-1}(\hat{\xb}_{0})\nabla f_1(\hat{\xb}_{0})}^2\\
    &+ \frac{4L^3}{\mu^2}\sum_{t=2}^T\norm{\xb^*_t-\xb^*_{t-1}}^2\\
    &=O(C^*_{2,T}),
\end{split}
\end{equation}
by noting that $\max_{t\in\{2,\ldots,T\}}\norm{\xb_t^* - \xb_{t-1}^*} \leq \frac{\mu}{2L_H}$.
\end{proof}

\begin{lemma}\label{L: Convergence of GD in iterates for strongly convex and smooth} (Theorem 2.1.14 in \citep{nesterov1998introductory})
Suppose that $f:\mathrm{R}^n\to \mathrm{R}$ is $\mu$-strongly convex and $L$-smooth. Letting $\xb^*=\argmin_{\xb\in\mathrm{R}^n}f(\xb)$ and $\vb=\ub-\eta\nabla f(\ub)$, where $0<\eta\leq \frac{2}{\mu+L}$, then we have
$$\norm{\vb-\xb^*}^2\leq \big(1-\frac{2\eta\mu L}{\mu+L}\big)\norm{\ub-\xb^*}^2.$$

\end{lemma}

\begin{proof}[\textbf{Proof of Theorem \ref{T:Online multiple gradient}}]
Observe that
\begin{equation}\label{eq: OMGD proof eq1}
\begin{split}
    f_t(\xb_t) - f_t(\xb^*_t) &= \big(f_t(\xb_t) - f_{t-1}(\xb_t)\big)\\
    &+ \big(f_{t-1}(\xb_t) - f_{t-1}(\xb^*_{t-1})\big)\\
    &+ \big(f_{t-1}(\xb^*_{t-1}) - f_t(\xb^*_t)\big).
\end{split}
\end{equation}
For the right hand side of \eqref{eq: OMGD proof eq1}, the first term is bounded as the following
\begin{equation}\label{eq: OMGD proof eq2}
    f_t(\xb_t) - f_{t-1}(\xb_t)\leq \sup_{\xb\in \Xc}|f_t(\xb) - f_{t-1}(\xb)|,    
\end{equation}
and since $\xb^*_{t-1}$ is the minimizer of $f_{t-1}$, we have 
\begin{equation}\label{eq: OMGD proof eq3}
    f_{t-1}(\xb^*_{t-1}) - f_t(\xb^*_t) \leq f_{t-1}(\xb^*_{t}) - f_t(\xb^*_t) \leq \sup_{\xb\in \Xc}|f_t(\xb) - f_{t-1}(\xb)|,
\end{equation}
where $\Xc$ in \eqref{eq: OMGD proof eq2} and \eqref{eq: OMGD proof eq3} is defined as the convex hull of $\{\xb_t,\xb^*_t\}_{t=1}^T$.\\\\
Based on the update rule of Algorithm \ref{alg:OMGD} and Lemma \ref{L: Convergence of GD in iterates for strongly convex and smooth}, we have 
\begin{equation}\label{eq: OMGD proof eq4}
    \norm{\xb_t - \xb^*_{t-1}}^2\leq \big(1-\frac{2\eta\mu L}{\mu+L}\big)^{K_t} \norm{\xb_{t-1} - \xb^*_{t-1}}^2\leq \big(1-\frac{2\eta\mu L}{\mu+L}\big)^{K_t} D^2,
\end{equation}
where $D\triangleq \max_{t=1,\ldots, T}\norm{\xb_t-\xb^*_t}$.\\\\
Then, for the second term of \eqref{eq: OMGD proof eq1}, based on smoothness and \eqref{eq: OMGD proof eq4}, we have
\begin{equation}\label{eq: OMGD proof eq5}
    f_{t-1}(\xb_t) - f_{t-1}(\xb^*_{t-1})\leq \frac{L}{2}\norm{\xb_t-\xb^*_{t-1}}^2\leq \frac{L}{2}\big(1-\frac{2\eta\mu L}{\mu+L}\big)^{K_t} D^2,
\end{equation}
Combining \eqref{eq: OMGD proof eq2}, \eqref{eq: OMGD proof eq3} and \eqref{eq: OMGD proof eq5} with \eqref{eq: OMGD proof eq1} and summing over $t$, we have 
\begin{equation}\label{eq: OMGD proof eq6}
\begin{split}
    \sum_{t=1}^T f_t(\xb_t)-f_t(\xb^*_t)&\leq (f_1(\xb_1)-f_1(\xb^*_1)) + 2\sum_{t=2}^T \sup_{\xb\in \Xc}|f_t(\xb) - f_{t-1}(\xb)|\\ &+ \sum_{t=2}^T\frac{D^2L}{2} \big(1-\frac{2\eta\mu L}{\mu+L}\big)^{K_t}\\
    &\leq (f_1(\xb_1)-f_1(\xb^*_1)) + 2V_T + \frac{D^2L}{2}\sum_{t=2}^T\frac{1}{t^2}\\
    &\leq (f_1(\xb_1)-f_1(\xb^*_1)) + 2V_T + \frac{D^2L}{2}(\frac{\pi^2}{6})\\
    &= O(V_T + 1),
\end{split}
\end{equation}
where the second inequality holds since $t^2\big(1-\frac{2\eta\mu L}{\mu+L}\big)^{K_t}\leq 1$ based on the choice of $K_t$.\\\\
On the other hand, from \eqref{eq: OMGD proof eq4} we have that
\begin{equation}\label{eq: OMGD proof eq7}
\begin{split}
    \norm{\xb_t - \xb^*_t}^2&\leq 2\norm{\xb_t - \xb^*_{t-1}}^2 + 2\norm{\xb^*_{t-1}-\xb^*_t}^2\\
    &\leq 2\big(1-\frac{2\eta\mu L}{\mu+L}\big)^{K_t}\norm{\xb_{t-1} - \xb^*_{t-1}}^2 + 2\norm{\xb^*_{t-1}-\xb^*_t}^2\\
    &\leq 2\big(1-\frac{2\eta\mu L}{\mu+L}\big)^{K_t} D^2 + 2\norm{\xb^*_{t-1}-\xb^*_t}^2.
\end{split}
\end{equation}
Summing \eqref{eq: OMGD proof eq7} over $t=2,\ldots,T$ and adding $\norm{\xb_1 - \xb^*_1}^2 $ to both sides, we have 
\begin{equation}\label{eq: OMGD proof eq8}
\begin{split}
    \sum_{t=1}^T \norm{\xb_t - \xb^*_t}^2&\leq \norm{\xb_1 - \xb^*_1}^2 + 2D^2\sum_{t=2}^T\big(1-\frac{2\eta\mu L}{\mu+L}\big)^{K_t}  + 2\sum_{t=2}^T\norm{\xb^*_t-\xb^*_{t-1}}^2\\
    &\leq \norm{\xb_1 - \xb^*_1}^2 + 2D^2(\frac{\pi^2}{6}) + 2\sum_{t=2}^T\norm{\xb^*_t-\xb^*_{t-1}}^2.
\end{split}    
\end{equation}
Based on smoothness and the fact that $\norm{\nabla f_t(\xb^*_t)}=0$, we know
$$\sum_{t=1}^T f_t(\xb_t) - f_t(\xb^*_t)\leq \sum_{t=1}^T\frac{L}{2}\norm{\xb_t-\xb^*_t}^2.$$
Substituting \eqref{eq: OMGD proof eq8} into the inequality above, we have
\begin{equation}\label{eq: OMGD proof eq9}
\begin{split}
    \sum_{t=1}^T f_t(\xb_t) - f_t(\xb^*_t)&\leq \sum_{t=1}^T\frac{L}{2}\norm{\xb_t-\xb^*_t}^2\\
    &\leq \frac{L}{2}\big[\norm{\xb_1 - \xb^*_1}^2 + 2D^2(\frac{\pi^2}{6}) + 2\sum_{t=2}^T\norm{\xb^*_t-\xb^*_{t-1}}^2\big]\\
    &= O(C^*_{2,T}+1).
\end{split}
\end{equation}
\end{proof}

\begin{proof}[\textbf{Proof of Theorem \ref{T: Improved OPGD dynamic bound (constrained)}}]
This proof mainly follows the proof for Theorem \ref{T: Improved OPGD dynamic bound}. Therefore, we would just point out some different steps here. Recall the update rule of Algorithm \ref{alg:OPGD(constrained)}: $\xb_{t+1} = \Pi_{\Xc}\bigg(\xb_t - \eta\Ab_t^{-1}\nabla f_t(\xb_t)\bigg)$. Letting $\yb_{t+1} = \xb_t - \eta\Ab_t^{-1}\nabla f_t(\xb_t)$, based on the Pythagorean theorem (see section 2.1.1 in \citep{hazan2019introduction}) and \eqref{eq:OPGD proof eq1}, we get
\begin{equation}\label{eq:OPGD (constrained) proof eq1}
\begin{gathered}
    (\xb_{t+1} - \xb_t^*)^\top\Ab_t(\xb_{t+1} - \xb_t^*) \leq (\yb_{t+1} - \xb_t^*)^\top\Ab_t(\yb_{t+1} - \xb_t^*) = \\ (\xb_{t} - \xb_t^*)^\top\Ab_t(\xb_{t} - \xb_t^*) - 2\eta\nabla f_t(\xb_t)^\top(\xb_{t} - \xb_t^*) + \eta^2\nabla f_t(\xb_t)^\top\Ab_t^{-1}\nabla f_t(\xb_t). 
\end{gathered}
\end{equation}

Following the same steps in the proof of Theorem \ref{T: Improved OPGD dynamic bound}, we get the inequality like \eqref{eq:OPGD proof eq6}
\begin{equation}\label{eq:OPGD (constrained) proof eq2}
\begin{split}
    \sum_{t=1}^T\nabla f_t(\xb_t)^\top (\xb_t - \xb_t^*)&\leq \frac{\eta}{2}\sum_{t=1}^T \nabla f_t(\xb_t)^\top\Ab_t^{-1}\nabla f_t(\xb_t) + \frac{1}{2\eta}(\lambda-\lambda^\prime(1-c))\sum^T_{t=1}\norm{\xb_t - \xb_t^*}^2\\ &+ \frac{\lambda^\prime}{2\eta}(\frac{1}{c}-1)\sum^T_{t=1}\norm{\xb_{t+1}^* - \xb_t^*}^2 + \frac{\lambda^\prime(1-c)}{2\eta}\norm{\xb_1-\xb_1^*}^2.
\end{split}
\end{equation}

Now, using the $L$-smoothness, $\nabla f_t(\xb_t)^\top\Ab_t^{-1}\nabla f_t(\xb_t)$ can be bounded as follows
\begin{eqnarray*}
    \nabla f_t(\xb_t)^\top\Ab_t^{-1}\nabla f_t(\xb_t) &=& (\nabla f_t(\xb_t)-\nabla f_t(\xb_t^*)+\nabla f_t(\xb_t^*))^\top\Ab_t^{-1}(\nabla f_t(\xb_t)-\nabla f_t(\xb_t^*)+\nabla f_t(\xb_t^*))\\
    &=&(\nabla f_t(\xb_t)-\nabla f_t(\xb_t^*))^\top\Ab_t^{-1}(\nabla f_t(\xb_t)-\nabla f_t(\xb_t^*))\\
    &+& 2(\nabla f_t(\xb_t)-\nabla f_t(\xb_t^*))^\top\Ab_t^{-1}\nabla f_t(\xb_t^*) - \nabla f_t(\xb_t^*)^\top\Ab_t^{-1}f_t(\xb_t^*)\\
    &\leq&(\nabla f_t(\xb_t)-\nabla f_t(\xb_t^*))^\top\Ab_t^{-1}(\nabla f_t(\xb_t)-\nabla f_t(\xb_t^*))\\
    &+& 2(\nabla f_t(\xb_t)-\nabla f_t(\xb_t^*))^\top\Ab_t^{-1}\nabla f_t(\xb_t^*)\\
    &\leq&\frac{L^2}{\lambda^\prime}\norm{\xb_t-\xb_t^*}^2 + \frac{2L}{\lambda^\prime}D\norm{\nabla f_t(\xb_t^*)},
\end{eqnarray*}
where $D=\max_{\xb,\yb\in\Xc}\norm{\xb-\yb}$.

Substituting above into \eqref{eq:OPGD (constrained) proof eq2} and rearranging it, we get
\begin{equation}\label{eq:OPGD (constrained) proof eq3}
\begin{split}
    &\sum_{t=1}^T\nabla f_t(\xb_t)^\top (\xb_t - \xb_t^*)
    -\bigg[\frac{\lambda-\lambda^\prime(1-c)}{2\eta}+\frac{\eta L^2}{2\lambda^\prime}\bigg]\sum_{t=1}^T\norm{\xb_t-\xb_t^*}^2\\\leq &\frac{\lambda^\prime}{2\eta}(\frac{1}{c}-1)\sum^T_{t=1}\norm{\xb_{t+1}^* - \xb_t^*}^2 + \frac{\lambda^\prime(1-c)}{2\eta}\norm{\xb_1-\xb_1^*}^2+\frac{\eta\lambda D}{\lambda^\prime}\sum_{t=1}^T\norm{\nabla f_t(\xb^*_t)}.
\end{split}
\end{equation}
Following the rest steps in the proof of Theorem \ref{T: Improved OPGD dynamic bound}, the theorem is proved.
\end{proof}

\begin{lemma}\label{L: Convergence of GD in iterates for strongly convex and smooth (constrained)} (Lemma 5 in \citep{zhang2017improved})
Suppose that $f:\Xc\to \mathrm{R}$ is $\mu$-strongly convex and $L$-smooth. Letting $\xb^*=\argmin_{\xb\in\Xc}f(\xb)$ and $\vb=\Pi_{\Xc}\bigg(\ub-\eta\nabla f(\ub)\bigg)$, where $0<\eta\leq \frac{1}{L}$, then we have
$$\norm{\vb-\xb^*}^2\leq \big(1-\frac{2\mu }{1/\eta+\mu}\big)\norm{\ub-\xb^*}^2.$$
\end{lemma}

\begin{proof}[\textbf{Proof of Theorem \ref{T:Online multiple gradient (constrained)}}]
This proof mainly follows the proof of Theorem \ref{T:Online multiple gradient}, so we would just point out some different steps here.\\
Instead of having \eqref{eq: OMGD proof eq4} and \eqref{eq: OMGD proof eq5}, we get the following two equations:\\
Based on the update rule of Algorithm \ref{alg:OMGD(constrained)} and Lemma \ref{L: Convergence of GD in iterates for strongly convex and smooth (constrained)}, we get 
\begin{equation}\label{eq: OMGD (constrained) proof eq1}
    \norm{\xb_t - \xb^*_{t-1}}^2\leq \big(1-\frac{2\mu }{1/\eta+\mu}\big)^{K_t} \norm{\xb_{t-1} - \xb^*_{t-1}}^2\leq \big(1-\frac{2\mu }{1/\eta+\mu}\big)^{K_t} D^2,
\end{equation}
where $D\triangleq \max_{\xb,\yb\in\Xc}\norm{\xb-\yb}$.\\
Based on the $L$-smoothness and \eqref{eq: OMGD (constrained) proof eq1}, we have 
\begin{equation}\label{eq: OMGD (constrained) proof eq2}
\begin{split}
    f_{t-1}(\xb_t) - f_{t-1}(\xb^*_{t-1})&\leq\nabla f_{t-1}(\xb^*_{t-1})^\top(\xb_t - \xb^*_{t-1})+ \frac{L}{2}\norm{\xb_t-\xb^*_{t-1}}^2\\
    &\leq\norm{\nabla f_{t-1}(\xb^*_{t-1})}D+ \frac{L}{2}\big(1-\frac{2\mu }{1/\eta+\mu}\big)^{K_t} D^2.
\end{split}
\end{equation}
Following the corresponding proof steps for Theorem \ref{T:Online multiple gradient}, we have
\begin{equation}\label{eq: OMGD (constrained) proof eq3}
\begin{split}
    \sum_{t=1}^T f_t(\xb_t)-f_t(\xb^*_t)&\leq (f_1(\xb_1)-f_1(\xb^*_1)) + 2\sum_{t=2}^T \sup_{\xb\in \Xc}|f_t(\xb) - f_{t-1}(\xb)|\\ &+ \sum_{t=2}^T\frac{D^2L}{2} \big(1-\frac{2\mu }{1/\eta+\mu}\big)^{K_t} + D\sum_{t=2}^T\norm{\nabla f_{t-1}(\xb^*_{t-1})}\\
    &\leq (f_1(\xb_1)-f_1(\xb^*_1)) + 2V_T + \frac{D^2L}{2}\sum_{t=2}^T\frac{1}{t^2} + D\sum_{t=2}^T\norm{\nabla f_{t-1}(\xb^*_{t-1})}\\
    &\leq (f_1(\xb_1)-f_1(\xb^*_1)) + 2V_T + \frac{D^2L}{2}(\frac{\pi^2}{6}) + D\sum_{t=2}^T\norm{\nabla f_{t-1}(\xb^*_{t-1})},
\end{split}
\end{equation}
where the second inequality holds since $t^2\big(1-\frac{2\mu }{1/\eta+\mu}\big)^{K_t}\leq 1$ based on the choice of $K_t.$\\
On the other hand, from \eqref{eq: OMGD (constrained) proof eq1} we have that
\begin{equation}\label{eq: OMGD (constrained) proof eq4}
\begin{split}
    \norm{\xb_t - \xb^*_t}^2&\leq 2\norm{\xb_t - \xb^*_{t-1}}^2 + 2\norm{\xb^*_{t-1}-\xb^*_t}^2\\
    &\leq 2\big(1-\frac{2\mu }{1/\eta+\mu}\big)^{K_t}\norm{\xb_{t-1} - \xb^*_{t-1}}^2 + 2\norm{\xb^*_{t-1}-\xb^*_t}^2\\
    &\leq 2\big(1-\frac{2\mu }{1/\eta+\mu}\big)^{K_t} D^2 + 2\norm{\xb^*_{t-1}-\xb^*_t}^2.
\end{split}
\end{equation}
Summing \eqref{eq: OMGD (constrained) proof eq4} over $t=2,\ldots,T$ and adding $\norm{\xb_1 - \xb^*_1}^2 $ to both sides, we have 
\begin{equation}\label{eq: OMGD (constrained) proof eq5}
\begin{split}
    \sum_{t=1}^T \norm{\xb_t - \xb^*_t}^2&\leq \norm{\xb_1 - \xb^*_1}^2 + 2D^2\sum_{t=2}^T\big(1-\frac{2\mu }{1/\eta+\mu}\big)^{K_t}  + 2\sum_{t=2}^T\norm{\xb^*_t-\xb^*_{t-1}}^2\\
    &\leq \norm{\xb_1 - \xb^*_1}^2 + 2D^2(\frac{\pi^2}{6}) + 2\sum_{t=2}^T\norm{\xb^*_t-\xb^*_{t-1}}^2.
\end{split}    
\end{equation}
Based on the $L$-smoothness, we know
\begin{eqnarray*}
\sum_{t=1}^T f_t(\xb_t) - f_t(\xb^*_t)&\leq& \sum_{t=1}^T\frac{L}{2}\norm{\xb_t-\xb^*_t}^2+\nabla f_t (\xb^*_t)^\top (\xb_t-\xb^*_t)\\
&\leq&\sum_{t=1}^T\frac{L}{2}\norm{\xb_t-\xb^*_t}^2+ D\sum_{t=1}^T\norm{\nabla f_t (\xb^*_t)}.
\end{eqnarray*}
Substituting \eqref{eq: OMGD (constrained) proof eq5} into the inequality above, we have
\begin{equation}\label{eq: OMGD (constrained) proof eq6}
\begin{split}
    \sum_{t=1}^T f_t(\xb_t) - f_t(\xb^*_t)&\leq \sum_{t=1}^T\frac{L}{2}\norm{\xb_t-\xb^*_t}^2+ D\sum_{t=1}^T\norm{\nabla f_t (\xb^*_t)}\\
    &\leq \frac{L}{2}\big[\norm{\xb_1 - \xb^*_1}^2 + 2D^2(\frac{\pi^2}{6}) + 2\sum_{t=2}^T\norm{\xb^*_t-\xb^*_{t-1}}^2\big]+ D\sum_{t=1}^T\norm{\nabla f_t (\xb^*_t)}.
\end{split}
\end{equation}

\end{proof}

\newpage
\bibliographystyle{apalike}
\bibliography{Reference}

\end{document}